\documentclass{article} 
\usepackage{nips12submit_e,times}
\usepackage[margin=1.1in]{geometry}  
\usepackage{graphicx}              
\usepackage{amsmath}              
\usepackage{amssymb}
\usepackage{amsfonts}              
\usepackage{amsthm}                
\usepackage{verbatim}		
\usepackage{algorithm}
\usepackage{algpseudocode}
\algtext*{EndWhile}
\algtext*{EndIf}
\algtext*{EndFor}

\numberwithin{figure}{section}

\newtheorem{thm}{Theorem}[section]

\newtheorem{prop}[thm]{Proposition}

\newtheorem{defn}[thm]{Definition}

\DeclareMathOperator{\polylog}{polylog}
\DeclareMathOperator{\poly}{poly}
\DeclareMathOperator*{\argmin}{arg\,min}
\DeclareMathOperator*{\argmax}{arg\,max}

\DeclareMathOperator{\sgn}{sgn}
\newcommand{\RR}{\mathbb{R}}      

\newcommand{\vnorm}[1]{\left\lVert#1\right\rVert} 
\newcommand{\ifn}{\mathbf{1}} 
\newcommand{\abs}[1]{\left| #1 \right|}

\newcommand{\prp}[2]{Pr_{#2} \left(#1\right)}

\newcommand{\err}[1]{\epsilon\left(#1\right)}
\newcommand{\iwerr}[1]{\hat{\epsilon}\left(#1\right)}

\newcommand{\cH}{\mathcal{H}}
\newcommand{\cX}{\mathcal{X}}

\newcommand{\cD}{\mathcal{D}}

\newcommand{\cV}{\mathcal{V}}
\newcommand{\cY}{\mathcal{Y}}

\newcommand{\cO}[1]{\mathcal{O}\left(#1\right)}
\newcommand{\ctO}[1]{\tilde{\mathcal{O}}\left(#1\right)}

\title{The Utility of Abstaining in Binary Classification}

\author{
Akshay Balsubramani \\
University of California, San Diego\\
\texttt{abalsubr@cs.ucsd.edu}
}

\nipsfinalcopy 

\begin{document}

\maketitle

\begin{abstract}
We explore the problem of binary classification in machine learning, with a twist - the classifier is allowed to abstain on any datum, professing ignorance about the true class label without committing to any prediction. This is directly motivated by applications like medical diagnosis and fraud risk assessment, in which incorrect predictions have potentially calamitous consequences. We focus on a recent spate of theoretically driven work in this area that characterizes how allowing abstentions can lead to fewer errors in very general settings. Two areas are highlighted: the surprising possibility of zero-error learning, and the fundamental tradeoff between predicting sufficiently often and avoiding incorrect predictions. We review efficient algorithms with provable guarantees for each of these areas. We also discuss connections to other scenarios, notably active learning, as they suggest promising directions of further inquiry in this emerging field.
\end{abstract}

\section{Introduction}
Consider a general practice physician treating a patient with unusual or ambiguous symptoms. The general practitioner often does not have the capability or equipment to confidently diagnose such an ailment, but will be able to refer the patient to a specialist or hospital. Therefore, the GP is faced with a difficult choice: either make a potentially erroneous diagnosis and act on it (which can occasionally have catastrophic consequences) or avoid committing to such a diagnosis and refer the patient on instead (which will almost certainly cost extra time and resources).

Such a situation motivates the study of prediction algorithms which are able not only to form a hypothesis about the correct prediction, but also \emph{abstain} entirely from making a prediction if they are not confident enough. In machine learning, these algorithms are often formulated as classifiers, and the field of classification with abstention is therefore of increasing research interest.

In this manuscript, we survey the problem of binary classification where the classifier is allowed to abstain. 


\subsection{Outline}

We will concern ourselves with the provable potential benefits of abstaining in two distinct scenarios. The first, covered in Section~\ref{kwiksec}, is an online learning model free of distributional assumptions on the data, similar to the classical Mistake Bound model~\cite{L88}. The other, discussed in Section~\ref{SLsec}, is a standard statistical learning model in which the algorithm is initially given a chance to train on a set of labeled data, and the data are assumed to be drawn IID (independently and identically distributed) from a fixed distribution.

The abstain option in general, and each of the aforementioned models in particular, is intimately related to the well-studied paradigm of active learning, which aims to minimize the amount of labeled data required by a learner.

One idea that will recur in the survey, and is common to both settings we study, is the fundamental tradeoff between abstaining at a low rate and achieving low error when predictions are made. The intuitive reason for this is simple. A classifier will abstain exactly when it is unsure about which label to predict, and such unsureness implies that it would err on such examples relatively often if forced to predict on them. Traditionally, the classifier is indeed forced to always predict; the possibility of doing better than this case is another spur to explore the utility of abstention.

Note that we avoid discussing Bayesian or explicitly distribution-dependent methods here~\cite{C70}, as they are outside the scope of the manuscript; we believe the work surveyed here is itself extremely general and provides a unique web of insights into the problem.

Due to the volume of the material covered here and the necessary concision of this manuscript, full proofs of results will normally not be provided; the interested reader is referred to the appropriate references. However, we will endeavor to make the treatment here as standalone as possible, and provide proof sketches and intuitions whenever they are instructive.

\subsection{Preliminaries}
\label{sec:Notation}

We consider a standard binary classification setting in which there exist unlabeled data points $x \in \cX$, each of which is associated to a label in $\cY \triangleq \{-1,+1\}$; the pair $(x, y) \in \cX \times \cY$ is referred to as a labeled data point. When the labels are deterministic, we can think of them as being generated by a map $h^* : \cX \mapsto \cY$, so that for all labeled data points $(x_i, y_i)$ we have $y_i = h^* (x_i)$.

The goal is to determine a classifier (a hypothesis) $h : \cX \mapsto \{-1, +1, \perp\}$, where the classifier can either predict one of the two admissible labels or output $\perp$, which is referred to as \emph{abstaining}. The latter can be interpreted as "don't know," a statement of ignorance about which label is correct without committing at all to either. In this case, the risk of predicting the label of the given datum incorrectly is deemed to be excessive. The understanding is that a lower penalty is paid with the $\perp$ prediction than an incorrect $\pm 1$ prediction, which we call a \emph{mistake}.

The learning algorithms we consider are given a hypothesis class $\cH$, which is a set of hypotheses $h : \cX \mapsto \cY$ that is used to constrain the set of possible solutions and for computational reasons. We will consider algorithms which choose a single $h \in \cH$, as well as algorithms which choose a weighted majority vote of all $h \in \cH$. We also follow standard terminology in calling problems with $h^* \in \cH$ \emph{realizable} or \emph{separable}. The non-realizable scenario is often referred to as the \emph{agnostic} case.

We defer more specific notation to the body of the manuscript, where it is presented as needed.


\section{Abstaining in the KWIK Model}
\label{kwiksec}

In this section, we describe a general model which has been influential in recent advances in the abstention literature. It will serve to provide useful insights into the possibilities afforded by abstention, even in a very general setting.

\subsection{KWIK: Formulation}

The KWIK ("Knows What It Knows") model~\cite{LLWS11} is a distribution-free way of formulating the online solvability of problems with \emph{zero} incorrect predictions in total. Algorithms which accomplish this must be self-aware enough to recognize \textit{all} points where they might make an incorrect prediction and output $\perp$ on them, and yet still predict (correctly) on a nontrivial number of examples. The name of the model refers to this self-awareness that it requires from algorithms, which distinguishes it from the otherwise similar Mistake Bound (MB) model.

We first review the salient features of the popular MB model as a useful reference point for further investigation.


\subsubsection{The Mistake Bound Model}
\label{mbmodel}
The Mistake Bound (MB) model~\cite{L88} for binary classification is a general framework for the study of online binary classification algorithms. In it, the learning algorithm receives unlabeled data points chosen one at a time in some arbitrary way (by an adversary in the worst case). At the beginning of round $t$, the algorithm has the hypothesis $h_t$, and sees an unlabeled point $x_t$. It then predicts a label $h_t (x_t) = \hat{y}_t \in \cY$, after which the true label $y_t$ is revealed. The algorithm then makes any necessary internal modifications to update its hypothesis $h_t$ to $h_{t+1}$. The next point $x_{t+1}$ is then chosen, and the process repeats.

In this model, results of interest bound the total number of mistakes $\abs{\{t : \hat{y}_t \neq y_t \}}$ made by the learning algorithm during the process, and are called \emph{mistake bounds}. Typically, the true labels are assumed to be generated by some hypothesis in $h^* \in \cH$, so that $y_t = h^* (x_t)$. In this setting, learnability is defined as follows. Here $dim(\cH)$ represents some predefined notion of the complexity of $\cH$ (for finite $\cH$, take $dim(\cH) = \abs{\cH}$).

\begin{defn}
\label{mbdefn}
An algorithm learns $\cH$ in the MB model if for any $\delta \in (0,1)$, the algorithm makes a total of at most $\poly (\frac{1}{\delta}, dim(\cH))$ mistakes with probability at least $1 - \delta$ over entire runs of the algorithm.
\end{defn}

Well-known examples of mistake bounds include those for the perceptron~\cite{KWA97} and weighted majority~\cite{LW89} learning algorithms.

Such results are considered quite strong, since they imply that the algorithm eventually is mistake-free and because each unlabeled datum $x_t$ may even be chosen by an adversary with full knowledge of the past $\{x_i, \hat{y}_i, y_i\}_{i=1}^{t-1}$. Guarantees in the standard setting of statistical learning, which we will consider more thoroughly in Section~\ref{SLsec}, can be seen to be less general primarily because they assume the unlabeled data $x_t$ to be drawn IID from a distribution that is fixed a priori.

The KWIK model, however, is a strict restriction of this MB model, as we will see next.


\subsubsection{The KWIK Model}


The KWIK~\cite{LLWS11} model is used to analyze online supervised learning algorithms. The setting is very similar to the MB setting, with the algorithm seeing an adversarially chosen point $x_t$ on the $t^{th}$ round. However, the algorithm predicts $\hat{y}_t \in \cY \cup \{\perp\}$, and only observes the true label when it predicts $\perp$.

As in the previous discussion, we assume realizability: $h^* \in \cH$. Then we can define our goal in this model, KWIK-learnability:
\begin{defn}
\label{kwikdefn}
An algorithm learns $\cH$ in the KWIK model if for any $\epsilon, \delta \in (0,1)$, the following are true with probability at least $1 - \delta$ over entire runs of the algorithm:
\begin{enumerate}
\item If $\hat{y}_t \neq \perp$, then $\abs{\hat{y}_t - y_t} < \epsilon$.
\item $\abs{\{t : \hat{y}_t = \perp\}}$, the total number of times the algorithm abstains, is at most $\poly (\frac{1}{\epsilon}, \frac{1}{\delta}, dim(\cH))$.
\end{enumerate}
\end{defn}

It is clear that this shares many similarities with Definition~\ref{mbdefn}. Some notes on Defn.~\ref{kwikdefn} are in order:
\begin{itemize}
\item The first requirement corresponds to the algorithm making no mistakes when it does choose to predict, which is the hallmark of this framework. (Note: As pointed out by Li et al.~\cite{LLWS11}, this does not make sense in the non-realizable case $h^* \notin \cH$. 
Therefore, the definition can be generalized to the non-realizable case by relaxing the first requirement, though we only focus on the realizable case which most known results address.)
\item The definition applies equally to regression problems (though we have presented $y_t$ as being a discrete classification label). Another generalization is to the case when instead of observing the true label $y_t$, the algorithm observes a possibly randomized quantity $z_t$ that depends on $y_t$. For instance, when learning the bias of a biased coin, $z_t$ is a Bernoulli random variable with bias $y_t$. We neglect to formalize this for simplicity, but the full statement is in~\cite{LLWS11}.
\end{itemize}

KWIK learnability is a fairly general condition for the same reasons discussed for MB learnability in Section~\ref{mbmodel}. In fact, it is stronger:
\begin{prop}
\label{KWIKtoMB}
Any KWIK algorithm which outputs $\perp$ at most $B$ times can be converted into an MB algorithm for the same problem which makes at most $B$ mistakes.
\end{prop}
\begin{proof}
The MB algorithm is the same as the KWIK algorithm, except that it outputs a prediction every time the KWIK algorithm would output $\perp$.
\end{proof}

As discussed earlier, this is due to the self-awareness required of the algorithm. An MB algorithm, by contrast, need only ensure it is learning correctly, as it blithely predicts without needing to consider its confidence in its own predictions.

For this reason, there are problems which are MB-learnable, but for which no KWIK algorithm exists; in a sense, the converse of Prop.~\ref{KWIKtoMB} does not hold. One such problem is learning a Boolean singleton: the hypothesis class is $h_i : \{0,1\}^n \mapsto \cY , \;i = 1,\dots,2^n$, predicting $+1$ iff the input is a binary representation of $i-1$. This can be MB-learned with at most 1 mistake by simply predicting $-1$ always until the mistake is made. However, KWIK-learning it is more difficult - the restriction that the algorithm must be \emph{absolutely sure} to predict is excessively strong, because in the worst case it is not sure until seeing $2^n - 1$ examples (effectively brute-force searching the space for the $+1$ data point). For the formal details, the reader is referred to Example 1 of~\cite{SZB10}.


The generality of the KWIK framework allows it to be used even for algorithms with adaptive sampling strategies, in which there are complex dependencies between the data points sampled. The framework was originally motivated by applications in reinforcement learning, where such situations arise regularly. It has proved a useful tool in proving optimality guarantees for such algorithms; a further discussion is beyond the scope of this manuscript, but the interested reader is referred to~\cite{LLWS11, DLL09, WSDL09, SS11}, and to~\cite{AADK13} for a recent application to the setting of multi-armed bandits.


\subsection{KWIK: Algorithmic Building Blocks}
As we have seen, the KWIK definition is extremely general. Remarkably, however, there are efficient KWIK algorithms for several nontrivial problems. These are for the most part derived from a few generic algorithms, which provide glimpses into the type of reasoning that is fruitful in this distribution-free setting. We therefore present some of these representative "building block" algorithms, along with their associated proof ideas.

\subsubsection{Enumeration}
\label{enumsec}

The most ubiquitous and central such method, the \emph{enumeration} algorithm, works in the realizable case when $\cH$ is finite, and is laid out as Algorithm~\ref{enumalg}.
\begin{algorithm}
\caption{Enumeration}
\label{enumalg}
\begin{algorithmic}[1]
\State $\cV \gets \cH$
\For{i = 1, 2, \dots}
	\If{$h(x_i)$ is the same $\forall h \in \cV$}
		\State Choose any $h \in \cV$ and predict $h(x_i)$
	\Else
		\State Output $\perp$ and receive label $y_i$
		\State $\cV \gets \{h \in \cV : h(x_i) = y_i \}$
	\EndIf
	\If{$\abs{\cV} = 1$}
		\State Terminate with final output $\cV$
	\EndIf
\EndFor
\end{algorithmic}
\end{algorithm}

The algorithm maintains a \emph{version space} $\cV$ of the hypotheses consistent with all data seen so far, paring down $\cV$ appropriately as more data are revealed. This is a very natural strategy for making conservative predictions. The optimal hypothesis $h^*$ is always in $\cV$ by construction; consequently, any predicted label $\hat{y}_i$ always agrees with $h^* (x_i)$, which is the true label $y_i$ by the realizability assumption.

Every time the algorithm abstains, at least one hypothesis in $\cV$ is predicting the wrong label, and such hypotheses are removed each iteration. Therefore, each abstention leads to the removal of at least one hypothesis from $\cV$, leading to a KWIK bound of $\abs{\cH} - 1$ abstentions. 

There is a natural correspondence between predicting $\perp$ here and simply requesting a label - abstention may be thought of as a couched label request. Given this and the simplicity of the enumeration procedure, it is not surprising to find that Algorithm~\ref{enumalg} appears independently in \emph{active learning}, a supervised learning paradigm where labels are only provided upon the learner's request~\cite{S12}. These and other active learning connections are discussed in Sec.~\ref{cssactive}.

Although the enumeration algorithm is appealingly general and simple, the major drawback is the same one confronted by the active learning literature: explicitly maintaining the version space by storing all its constituent hypotheses can be computationally intractable, requiring an exponential amount of memory!

Fortunately, as we will observe later in the manuscript, there are nontrivial special cases for which $\cV$ can be implicitly maintained so that the operations of Algorithm~\ref{enumalg} can be performed without actually storing $\cV$, even when $\cH$ is infinite.


\subsubsection{Other Basic KWIK Algorithms}

In order to highlight the utility of KWIK, we will briefly touch upon a few other basic algorithms in this framework. All of them are efficient - they run in polynomial time.

\begin{itemize}
\item \textit{Learning the bias of a coin}: The algorithm simply abstains on the first $T$ examples to get IID Bernoulli-distributed labels in order to estimate the bias, and thereafter predicts on every example according to the empirical estimate of the bias. Applying the Hoeffding bound yields a KWIK bound of $T \in \cO{\epsilon^{-2} \ln (1/\delta)}$.
\item \textit{Learning the distribution of an $n$-sided die}: This is the multinomial-distribution analogue of the above coin-learning problem. It can be solved by reduction to $n$ "one-vs.-rest" coin-learning problems, and the algorithm is the same as for the coin-learning case. A Chernoff bound analysis yields the KWIK bound $\cO{n \epsilon^{-2} \ln (n/\delta)}$.
\item \textit{Linear regression with no noise in $\RR^n$}: This is KWIK-learnable by a beautifully simple algorithm that exploits linear algebraic structure. The algorithm maintains a running training set of points $x$ and their corresponding linear function values $f(x) = \theta \cdot x,\;\theta \in \RR^n$ (unknown $\theta$). For any new point $x_t$, the algorithm first checks if $x_t$ is in the linear span of the points in the training set so far, which can be done efficiently by solving a system of linear equations. If so, $f(x_t)$ is an efficiently determinable linear combination of the function values in the training set, and the algorithm can predict with certainty. Otherwise, it abstains and adds $(x_t, f(x_t))$ to the training set. The training set cannot grow beyond size $n$, so the KWIK bound is $n$.
\item \textit{Linear regression with additive white noise in $\RR^n$}: Though neither the proof nor the algorithm are as simple as for the zero-noise case, the algorithm predicts using the least-squares solution, or abstains by assessing a confidence measure based on the eigenvalue spectrum of the data. The KWIK bound is $\ctO{n^3 \epsilon^{-4}}$. (Note: This setting has been studied independently for a much weaker \emph{oblivious}, rather than adaptive, adversary; under this assumption, a much better bound of $\ctO{n \epsilon^{-2}}$ is proved in~\cite{CGO09}.)
\item Algorithms also exist for certain generic combinations of KWIK-learnable scenarios. For instance, if $\cH_i ,\;i = 1,\dots,n$ are KWIK-learnable, then so is $\cup_i \cH_i$, even when observations are corrupted by white noise. Also KWIK-learnable are hypothesis classes whose domains and ranges are respectively Cartesian products of the domains and ranges of KWIK-learnable classes $\cH_i$.
\end{itemize}

The intent here is to highlight that KWIK is a nontrivial model despite its extreme generality. We will now address a central issue of this survey - the tradeoff between errors and abstentions - by exploring the interesting consequences of allowing the learner to make some mistakes.

\subsection{Allowing a Few Mistakes in KWIK}
\label{kwiksomemistakes}

Though KWIK has proved to be quite useful in the past few years, its restriction that the learner make absolutely no mistakes is somewhat extreme in many circumstances. As we will see, allowing a few mistakes in the framework can drastically improve abstention rates.


\subsubsection{Enumeration With a Few Mistakes}

The cornerstone enumeration method of Algorithm~\ref{enumalg} has recently been extended to allow up to $k$ mistakes for some prespecified parameter $k$, again in an appealingly simple and generic manner~\cite{SZB10}. This newer algorithm, which we dub \emph{relaxed enumeration} in our presentation, is outlined in Algorithm~\ref{relenumalg}, and again is motivated by the realizable case for finite $\cH$.

(Note: Here, the true label of a point is revealed regardless of whether the algorithm abstains or not. This is still consistent with the original KWIK framework of Algorithm~\ref{enumalg} because of that algorithm's zero-mistakes property.)

\begin{algorithm}
\caption{Relaxed Enumeration}
\label{relenumalg}
\begin{algorithmic}[1]
\State $\cV \gets \cH, s \gets \abs{\cH}^{\frac{k}{k+1}}, m \gets 0$
\For{i = 1, 2, \dots}
	\State $\mu \triangleq \displaystyle \min_{\lambda \in \{-1, +1\}} \abs{\{h \in \cV : h(x_i) = \lambda \}}$, so that $\mu$ is the number of hypotheses in $\cV$ predicting the minority label
	\If{$\mu \leq s$}
		\State Predict with the majority, i.e. predict $\hat{y}_i \triangleq \displaystyle \argmax_{\lambda \in \{-1, +1\}} \abs{\{h \in \cV : h(x_i) = \lambda \}} $
	\Else
		\State Output $\perp$
	\EndIf
	\State Receive label $y_i$
	\State $\cV \gets \{h \in \cV : h(x_i) = y_i \}$
	\If{$\hat{y}_i \neq y_i$ and $\hat{y}_i \neq \perp$}
		\State $m \gets m+1$
		\State $s \gets \abs{\cV}^{\frac{k-m}{k+1-m}}$
	\EndIf
	\If{$\abs{\cV} = 1$}
		\State Terminate with final output $\cV$
	\EndIf
\EndFor
\end{algorithmic}
\end{algorithm}

Algorithm~\ref{relenumalg} has an instructive interpretation as a generalization of Algorithm~\ref{enumalg}. Like its predecessor, it tracks a version space $\cV$ and terminates when enough data have arrived to pare this down to a singleton set containing $h^*$. When an unlabeled datum $x_i$ arrives, a vote is taken over the hypotheses in $\cV$. The margin of the vote is used as a measure of confidence in deciding whether or not to predict. The threshold $s$ that is used to make this decision is decreased in such a way that at most $k$ mistakes are made; this is done by keeping track of the number of mistakes $m$ made so far.

Due to the tolerance for a few mistakes, the KWIK bound on the number of $\perp$s is far better than that of Algorithm~\ref{enumalg}. The proof uses an elegant inductive argument which is worth noting in full:
\begin{prop}[~\cite{SZB10}]
\label{relenumbound}
Algorithm~\ref{relenumalg}, in the realizable case, makes at most $k$ mistakes and outputs $\perp$ at most $(k+1)\abs{\cH}^{\frac{1}{k+1}}$ times.
\end{prop}

\begin{proof}
The proof proceeds by induction in $k$. The first mistake made reduces $\abs{\cV}$ so that $\abs{\cV} < \abs{\cH}^{\frac{k}{k+1}}$. Since the effective size of the hypothesis class being considered at any point is $\abs{\cV}$, the inductive hypothesis is that after the first mistake, the algorithm will terminate after at most $k-1$ mistakes and $k \abs{\cV}^{\frac{1}{k}} < k \left( \abs{\cH}^{\frac{k}{k+1}} \right)^{\frac{1}{k}} = k \abs{\cH}^{\frac{1}{k+1}}$ abstentions.

The base case of the induction concerns the number of abstentions before the first mistake. In this regime, $s = \abs{\cH}^{\frac{k}{k+1}}$, and $\mu > s$, so there are at least $\abs{\cH}^{\frac{k}{k+1}}$ hypotheses removed from $\cV$ every iteration for predicting the wrong label. Since $\abs{\cV} = \abs{\cH}$ at first, the algorithm abstains at most $\frac{\abs{\cH}}{\abs{\cH}^{\frac{k}{k+1}}} = \abs{\cH}^{\frac{1}{k+1}}$ times before the first mistake.

Combining this with the inductive hypothesis, the total number of $\perp$s is $\leq k \abs{\cH}^{\frac{1}{k+1}} + \abs{\cH}^{\frac{1}{k+1}} = (k+1)\abs{\cH}^{\frac{1}{k+1}}$, which completes the induction.
\end{proof}

In the original paper~\cite{SZB10}, the analysis follows that of the "Egg Dropping Game," a common brainteaser, which can be found in~\cite{GF08}. It appears to this author that the analysis of~\cite{SZB10} can in fact be significantly tightened and enhanced using further results in~\cite{GF08}, which utilize a more complex inductive argument. As this is unpublished, however, we limit ourselves to pointing it out as an open problem.

In any case, it is apparent that even for small $k$, allowing a few mistakes leads to vast improvements over Algorithm~\ref{enumalg} in the dependence on $\abs{\cH}$ - from $\cO{\abs{\cH}}$ to $\cO{\abs{\cH}^{\frac{1}{k+1}}}$. For instance, allowing just $k=1$ mistake yields a bound of $2\sqrt{\abs{\cH}}$. 

This again highlights one of the most fundamental ideas in this survey - the tradeoff between predicting sufficiently often and avoiding prediction errors. Though zero-error learning is significant in its own right, it is a rather extreme condition that implies a similarly extreme cost of a mistake relative to the cost of abstaining. Most applications would be content to allow a few mistakes in return for a significantly more useful classifier (one that abstains less). This promise is exactly what Prop.~\ref{relenumbound} quantifies.


\subsubsection{Tractable Relaxed Enumeration for Linear Separators}

Algorithm~\ref{relenumalg} in general suffers from the same computational intractability as Algorithm~\ref{enumalg}, because it too appears to require the version space to be explicitly stored. However, this can be circumvented in useful special cases, including when learning monotone disjunctions and linear separators~\cite{SZB10}. We discuss the intuition for the latter, as it is instructive; for formal proofs, refer to~\cite{SZB10}.

Concretely, let $\cX = \RR^d$ be the data space with the Euclidean norm, and let $\cH$ be the set of linear separators $\{w \in \RR^d : \vnorm{w} = 1\}$, so that the prediction rule is $h_w (x) = \sgn(w \cdot x)$. Denote the optimum separator by $w^*$. Assume the data are separable with margin $\gamma \triangleq \min_x \frac{w^* \cdot x}{\vnorm{x}} > 0$ - this loosely corresponds to the realizability assumption of relaxed enumeration. This is the setup used for the well-known standard perceptron algorithm~\cite{KWA97}. If points $x_1, \dots, x_n$ have already been seen, it will also be useful to think of the problem as a linear program with $d$ variables (one for each coordinate of $w$) and $n$ halfspace constraints $w \cdot x_i > 0$ for $\{i : y_i = +1\}$ and $w \cdot x_i < 0$ for $\{i : y_i = -1\}$.

~\cite{SZB10} show how to efficiently implement Algorithm~\ref{relenumalg} here using the following insight: the algorithm does not need to store $\cV$, because its decision-making process only requires it to deal with $\abs{\cV}$ and estimate $\abs{\{h \in \cV : h(x_i) = \pm 1 \}}$.

The first quantity, $\abs{\cV}$, is the volume of the feasible set of the linear program (LP) mentioned above, and it contracts significantly whenever a mistake is made. The $\gamma$-separability assumption implies that this feasible set cannot shrink below a certain volume, as there is some room around $w^*$ where consistent separators can fall. Combining these two facts controls the number of mistakes made, and finishes the use of $\abs{\cV}$.

However, the algorithm still needs to estimate the quantities $\abs{\{h \in \cV : h(x_i) = \pm 1 \}}$ for arbitrary $x_i$ to decide whether or not to abstain. This is essentially the problem of estimating the relative volumes of two convex sets. Using Hit-and-Run~\cite{LV06}, a tool for uniformly sampling from convex bodies, we can simply sample uniformly from the feasible set of the current LP, and see whether the sampled point is on the $+1$ or $-1$ side. Repeatedly sampling points essentially extends this idea to a Monte Carlo estimate of the relative volumes of $\abs{\{h \in \cV : h(x_i) = +1 \}}$ and $\abs{\{h \in \cV : h(x_i) = -1 \}}$, which the algorithm can use to make its decision.

It turns out that the number of points required can be managed to yield a polynomial-time algorithm which is equivalent to Algorithm~\ref{relenumalg} with high probability. Formalizing all this leads to the main result, which nicely complements the well-known perceptron mistake bound of $\cO{\gamma^{-2}}$~\cite{S-S12}:
\begin{prop}[~\cite{SZB10}]
Assume data are linearly $\gamma$-separable in $\RR^d$, and define $R = \left( 1 + \frac{\sqrt{d}}{\gamma} \right)^d $. Then for any $k > 0$, a linear separator can be learned with at most $k$ mistakes and $\cO{R^{1/k} \log R}$ abstentions.
\end{prop}

Omitting logarithmic factors, this bound on $\perp$ is $\approx \cO{\left( \frac{\sqrt{d}}{\gamma} \right)^{d/k}}$ for small $\gamma$. This implies the appealing fact that the mistake tolerance $k$ needs only to scale linearly in the dimensionality $d$ to achieve a fixed-exponent polynomial dependence on $\gamma$.

Note that although Algorithms~\ref{enumalg} and~\ref{relenumalg} were formulated for finite $\cH$ with the counting measure for simplicity, we have just sketched a tractable implementation for an infinite $\cH$ with the Lebesgue measure. This type of efficient Monte Carlo sampling can widely broaden the potential usability of the algorithms; for instance, linear separators constitute a very useful $\cH$ in applications.


\section{Abstaining in Statistical Learning}
\label{SLsec}

We now change course somewhat to examine the effect of the abstain option in a standard \emph{statistical learning} setting, where unlabeled data points $x_i$ are drawn IID from a distribution that is fixed beforehand. This can be thought of as a variation of the PAC (Probably Approximately Correct) model which is standard in learning theory.

As mentioned earlier, this setting is more restrictive than the KWIK model. It imposes an extra IID assumption on the data generation process, and also treats the batch case in which we begin tracking the algorithm's performance after it has already seen $m$ labels, rather than penalizing the algorithm for performance at the start of the learning process as well. As might be expected, these more restrictive assumptions lead to a wider range of stronger results than are known for the KWIK case. These results will now be introduced.

Formally, we now consider a standard batch learning setting where we have a training set $S = \{(x_1, y_1), \dots, (x_m, y_m)\}$ of points drawn i.i.d. from a distribution $\cD$. The learning algorithm trains on these points and thereafter sees no more labels. To emphasize the point, the learner will never even see the labels of points on which it predicts $\perp$ (recall that this was the case in KWIK).

Define the error of a hypothesis $\err{h} = Pr_\cD (h(x) \neq y)$; this is analogous to the number of mistakes in the KWIK framework. Similarly, rather than speaking of the total number of abstentions, the appropriate quantity to discuss will be the probability of abstentions over $\cD$. We also will speak of the training error (error on $S$) $\iwerr{h} = \frac{1}{m} \sum_{i=1}^m \ifn(h(x_i) \neq y_i)$.

We will otherwise follow the pattern laid out in our KWIK discussion of Section~\ref{kwiksec}, presenting results first for the zero-error case and then addressing the error-abstention tradeoff.

\subsection{An Algorithm for Zero-Error Learning}
As we did in the KWIK overview of Section~\ref{kwiksec}, we begin with the case when the algorithm is not allowed to make any erroneous predictions at all, and attempt to minimize the abstain probability. It is a nontrivial (and recent) result that this is even possible in a general setting. However, we will see how it can be done, and use the resulting algorithm as a point of departure to explore some powerful recent results in this vein.

Recall that this zero-error scenario's KWIK analogue (no mistakes) was addressed by Algorithm~\ref{enumalg}, the enumeration algorithm. It turns out that an almost identical algorithm can be used for the realizable case in the statistical learning setting as well. This was introduced in recent work as the Consistent Selective Strategy (CSS)~\cite{EW10}, described in Algorithm~\ref{cssalg}.

\begin{algorithm}
\caption{Consistent Selective Strategy (CSS)}
\label{cssalg}
\begin{algorithmic}[1]
\State Given: training set $S = \{(x_1, y_1), \dots, (x_m, y_m)\}$
\State $\cV \gets \{h \in \cH : h(x_j) = y_j \;\;\forall j = 1, \dots, m \}$
\For {any unlabeled point $x$}
\If{$h(x)$ is the same $\forall h \in \cV$}
	\State Choose any $h \in \cV$ and predict $\hat{y} = h(x)$
\Else
	\State Output $\perp$
\EndIf
\EndFor
\end{algorithmic}
\end{algorithm}

CSS is essentially a batch version of Algorithm~\ref{enumalg}, the enumeration KWIK algorithm for the realizable case. It has zero error with certainty for the same reason as mentioned for Algorithm~\ref{enumalg} - by construction, $h^* \in \cV$, and so any prediction must agree with $h^*$ and therefore be correct. 

As with Algorithm~\ref{enumalg}, though, the significance of CSS lies in the fact that it does not abstain too often:
\begin{prop}[Thm. 8,~\cite{EW10}]
For Algorithm~\ref{cssalg},
$$ \prp{\perp}{\cD} \leq \frac{1}{m} \left( (\ln 2) \min(\abs{\cH}, \abs{\cX}) + \ln \frac{1}{\delta} \right)$$
\end{prop}
\begin{proof}[Proof Sketch]
Two hypotheses are equivalent if they have the same predictions on all points in $\cX$, so there are effectively $K = 2^{\min(\abs{\cH}, \abs{\cX})}$ possible subsets of $\cH$. Of these, the proof considers only the $k \leq K$ which are \emph{not} unanimous on all the data (only on at most $1 - \epsilon$ of it, for some $\epsilon > 0$). For any such subset (call it $G_i , \;i \in \{1,\dots,k\}$), it is very likely that not all $m$ points in the training set are drawn from $G_i$. This implies that $\cV$ is different from $G_i$.

In this way, it can be argued that $\cV$ is different from all $G_i$. Since the $G_i$ are characterized only by being unanimous on $\leq 1 - \epsilon$ of the data, $\cV$ must be unanimous on $\geq 1 - \epsilon$ of the data. Since we predict exactly when $\cV$ is unanimous, we must predict quite often, so we cannot abstain very often.
\end{proof}

The $\Theta(1/m)$ dependence of the abstain probability establishes CSS as a nontrivial algorithm. In fact, there is a significant sense in which it is optimal!
\begin{prop}[Thm. 7,~\cite{EW10}]
\label{cssoptimality}
In the realizable case, any classifier which has zero error with certainty on any $\cD$ with any $h^*$ satisfies the following condition: if it predicts (a non-$\perp$ label) on a point $x$, then so does CSS. Consequently, there is no such classifier with $\prp{\perp}{\cD}$ less than CSS.
\end{prop}
\begin{proof}[Proof Sketch]
The argument is given for strategies which do not randomize the choice of whether to abstain on a given point. The proof is by contradiction, so assume there is a classifier $A$ that has zero error for any data distribution, and a point $x_0 \notin S$ such that $A(x_0) \neq \perp$ and $CSS(x_0) = \perp$. By the construction of CSS, there is a hypothesis $h_2 \in \cV$ such that $A (x_0) \neq h_2 (x_0)$.

Having fixed these facts, we now formulate a new classification problem (using the same $\cH$ and $\cX$) to prove the contradiction. Consider a new distribution $\cD_n$ over $\cX$ which puts positive probability on each of the $m$ points in $S$ and on $x_0$, and zero probability on all other points in $\cX$. Also define the new label-generating hypothesis $h^*_n \triangleq h_2$. Then with positive probability (say $\delta > 0$), an $m$-element training set drawn from $\cD_n$ will be equal to $S$. Now $A$ is deterministic and only depends on the training set, so $A$ will still predict on $x_0$ with probability $\delta$. But when it predicts on $x_0$, it will be erroneous because we established that $A (x_0) \neq h_2 (x_0) = h^*_n (x_0)$.

So $A$ does not have zero error with certainty, and we have our contradiction.

The argument for methods which randomize the abstain decision follows along the same lines.
\end{proof}

Of course, Prop.~\ref{cssoptimality} does not preclude the existence of a classifier which abstains less than CSS but only has zero error with high probability. Nevertheless, the result establishes the conceptual importance of CSS as the prototypical zero-error learning algorithm in statistical learning. It is particularly notable that the argument carries over to randomized algorithms as well - all of them are trumped in this setting by the deterministic CSS.

This author is not aware of this optimality argument being made for the variety of other extant CSS-like algorithms, though the logic is potentially extensible to these algorithms. It could be a particularly useful perspective to mention in the KWIK case for the enumeration algorithm, for which a potential combinatorial extension of this result is currently being examined by this author as a byproduct of the investigations of this manuscript.


\subsection{Zero-Error Learning in the Infinite-$\cH$ Case}

Interestingly, these favorable properties of CSS hold only for finite hypothesis classes; in fact, it is impossible in general for CSS to avoid abstaining everywhere when $\cH$ is infinite, even in the realizable case. The counterexample construction used to show this is of independent interest, and will now be outlined.

The construction is schematically shown in Figure 3.1. $\cH$ is the class of linear separators in $\RR^2$, and the data are uniformly distributed on two nonintersecting arcs, with one containing only positive examples and the other only negative examples as shown in the figure. It is clear that this is a realizable problem, because any separator lying between the arcs that does not intersect either one will have zero error. However, for any finite training set the CSS version space $\cV$ will only be unanimous for data in a convex polygon-like region inscribed in each arc; an example is the two shaded regions in Figure 3.1. Unfortunately, these regions only contain a finite number of data points which collectively have measure zero. So although CSS will predict correctly where it essays a prediction, it becomes a trivial algorithm which will almost surely abstain.

\begin{figure}[h]
\label{infhypexample}
\centering
\includegraphics[width=0.5\textwidth]{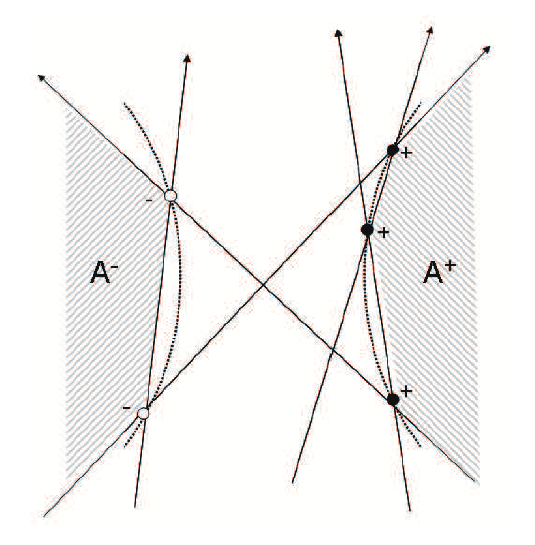}
\caption{Construction for which nontrivial CSS zero-error learning is impossible (from~\cite{EW10}).}
\end{figure}

This construction is a valuable one because similar arguments to the above make it a "hard" situation for many algorithms that depend on version space constructions and/or unanimity between hypotheses, which arise in other fields like active learning~\cite{DHM07}. Loosely speaking, it has the property that $\cH$ intuitively does not "fit" the data as represented; linear separators are not a natural choice for data whose intrinsic geometry is nonlinear. We will revisit this construction briefly later.


\subsection{Characterizing the Abstain-Error Tradeoff}
\label{tradeoffchar}
We now relax the requirement of zero error from the abstaining classifier and address how abstention rates can be reduced if a low error rate is allowed. The discussion here is the IID-setting analogue to that in Sec.~\ref{kwiksomemistakes}.

The tradeoff is illustrated in Figure 3.2, in which it is termed the \emph{risk-coverage} tradeoff. Coverage is simply the probability that an algorithm predicts rather than abstaining, and therefore is exactly the complement of the abstain probability $\prp{\perp}{\cD}$ we have used so far. Risk is a generalization of error probability to arbitrary loss functions. For our purposes hereafter, it can be taken as synonymous with error probability.

\begin{figure}[h]
\label{rctradeoff}
\centering
\includegraphics[width=0.85\textwidth]{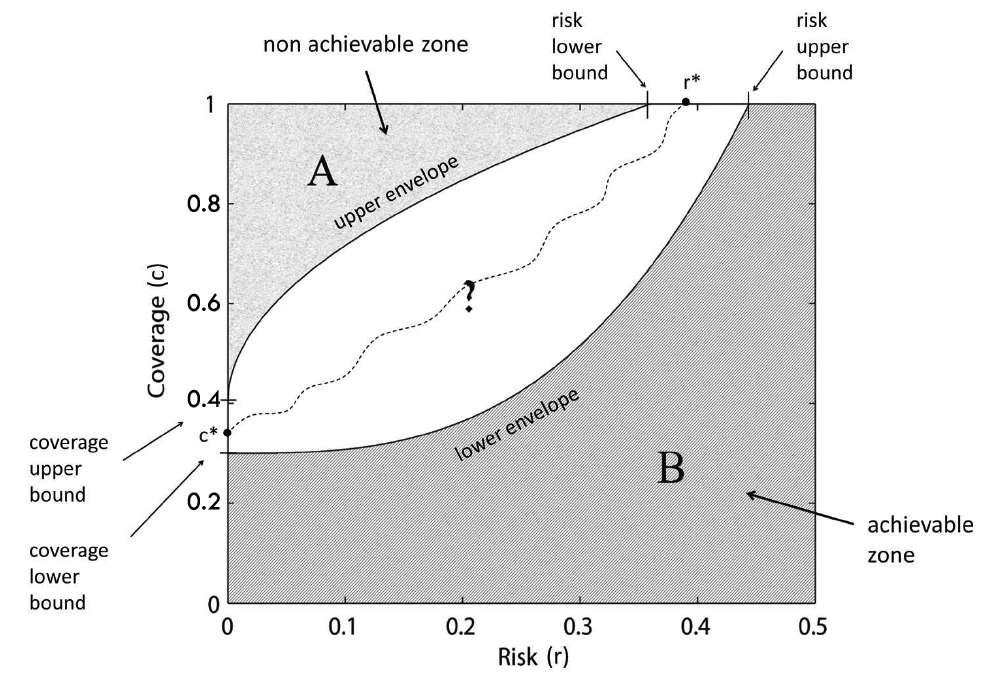}
\caption{The risk-coverage tradeoff in statistical learning (from~\cite{EW10}).}
\end{figure}

The tradeoff for any algorithm is schematically shown by the dotted curve in the figure: lower risk can only be achieved at the expense of coverage. The upper axis (with $r^*$) represents the case when the algorithm is not allowed to abstain; this is the scenario in standard binary classification, and $r^*$ represents the risk of the algorithm in this case. The left axis represents the zero-error case that CSS addresses, and $c^*$ is the coverage here. Most interestingly, upper and lower "envelopes" on this tradeoff graph can be proved for abstaining classifiers, setting limits on the risk-coverage tradeoff in the statistical learning setting.

We present the upper envelope only, as it is the more fundamental result as proved in~\cite{EW10}.

\begin{prop}[Thm. 37 from~\cite{EW10}]
Suppose $d$ is the VC dimension of $\cH$, and fix $\delta \in [0, \frac{1}{4}]$. Consider any classifier $A$ with the abstain capability, and suppose it has abstain probability $\rho$ and has error rate $R(A)$ on \emph{examples on which it predicts}. Then there is a data distribution $\cD$ such that with probability $\geq \delta$,
$$ R(A) \geq \min \left( \frac{1}{2} - \frac{1}{4 \rho} ,\;\;\frac{1}{2} - \frac{1}{2 \rho} + \frac{1}{16 \rho m} \left( d + \frac{16}{3} \ln (1-2\delta) \right) \right) $$
\end{prop}

No proof is given here for space reasons, but the reader is referred to~\cite{EW10} to examine the interesting $\cD$ construction and the associated VC dimension argument central to the proof.

The lower envelope proof involves interpolating between the two extremes of standard learning and CSS by considering a CSS-like classifier that predicts when $\cV$ is unanimous, but flips a coin of bias $\alpha \in [0,1]$ to decide whether to predict even when $\cV$ is not unanimous. (Clearly, the $\alpha = 0$ case corresponds to CSS and $\alpha=1$ to standard learning). The lower envelope is proved only for this $\alpha$-parametrized classifier in~\cite{EW10}.


\subsection{Relaxing CSS for the Non-Realizable Case}
Though we have characterized the fundamental abstention tradeoff, our presentation so far has lacked an algorithm which, like Algorithm~\ref{relenumalg} in the KWIK setting, exploits the tradeoff. We use this section and Section~\ref{classifierAgg} to explore two such methods in turn.

The first method we outline is directly inspired by CSS, and generalizes it to the non-realizable case. Here the target hypothesis $h^* = \displaystyle \argmin_{h \in \cH} \err{h}$; the difference from the realizable case is that $\err{h^*} \neq 0$ in general. The logic of CSS fails here, because $h^*$ is not necessarily in the version space $\cV$ as calculated on the training set $S$; indeed, it is possible that $\cV = \emptyset$. However, the IID assumption on the data means that $S$ is fairly representative of the data distribution $\cD$. Therefore, if we relax the definition of $\cV$ to include all low-training-error hypotheses (rather than zero-training-error hypotheses), we can trust that since $\iwerr{h^*} \approx \err{h^*}$, $h^*$ will be in this newly defined $\cV_{relaxed}$.

This is precisely what is done in~\cite{WE11}. For convenience, we dub the resulting method "Relaxed CSS." It is given in Algorithm~\ref{relcssalg}.

\begin{algorithm}
\caption{Relaxed CSS}
\label{relcssalg}
\begin{algorithmic}[1]
\State Given: training set $S = \{(x_1, y_1), \dots, (x_m, y_m)\}$, algorithm failure probability $\delta$, VC dim. of $\cH = d$
\State $\hat{h}^* \gets \argmin_{h \in \cH} \iwerr{h}$
\State $\cV_{relaxed} \gets \left\{h \in \cH : \iwerr{h} \leq \iwerr{\hat{h}^*} + 4 \sqrt{2 \frac{d \ln (2me/d) + \ln (8/\delta)}{m}} \right\}$
\For {any unlabeled point $x$}
\If{$h(x)$ is the same $\forall h \in \cV_{relaxed}$}
	\State Choose any $h \in \cV_{relaxed}$ and predict $\hat{y} = h(x)$
\Else
	\State Output $\perp$
\EndIf
\EndFor
\end{algorithmic}
\end{algorithm}

This algorithm is shown in~\cite{WE11} to satisfy good error and coverage bounds (with probability at least $1 - \delta$ over the choice of $S$) if the "excess error" $\ifn(h(x) \neq y) - \ifn(h^*(x) \neq y)$ is well-behaved over $\cD$ for all $h \in \cH$. More precisely, its error on the data on which it predicts is at most the error of $h^*$ on those data, with high probability. The full results are in~\cite{WE11}; we omit them here due to space constraints in favor of discussing an interesting implementation trick which provides further connections to other work.

Algorithm~\ref{relcssalg}, like CSS and both flavors of enumeration, faces the implementation difficulty of tracking a potentially exponentially large $\cV_{relaxed}$ so that it can detect unanimity or lack thereof. However, this task can be reduced to a suitably chosen empirical error minimization, as indicated in~\cite{WE11}:
\begin{prop}[Paraphrase of Lemma 6.1 in~\cite{WE11}]
For any $x \in \cX$, define $\hat{h}^*$ as in the algorithm and $\displaystyle \tilde{h}_x \triangleq \argmin_{h \in \cH} \{ \iwerr{h} \mid h(x) \neq \hat{h}^* (x) \}$. Then Algorithm~\ref{relcssalg} outputs $\perp$ iff $\iwerr{\tilde{h}_x} - \iwerr{\hat{h}^*} \leq 4 \sqrt{2 \frac{d \ln (2me/d) + \ln (8/\delta)}{m}}$.
\end{prop}
\begin{proof}
Working backwards, the condition $\iwerr{\tilde{h}_x} - \iwerr{\hat{h}^*} \leq 4 \sqrt{2 \frac{d \ln (2me/d) + \ln (8/\delta)}{m}}$ is equivalent to $\tilde{h}_x \in \cV_{relaxed}$. By construction, $\tilde{h}_x (x) \neq \hat{h}^* (x)$. Both hypotheses are in $\cV_{relaxed}$, so $\cV_{relaxed}$ is not unanimous on $x$, and the algorithm therefore abstains.
\end{proof}
This simple argument reduces the generally intractable task of storing a version space to the task of empirical error minimization, which lacks the space complexity of the former and can be done efficiently in many cases (including for infinite $\cH$). For this reason, it appears in several version space-based algorithms (including CSS, for which a variant is presented in~\cite{EW10}). This is discussed in Section~\ref{activenonzeroerror}.


\subsection{Aggregating Classifiers to Exploit the Tradeoff}
\label{classifierAgg}

We now turn our attention to a conceptually elegant algorithm that exploits the tradeoff in the statistical learning setting by taking a \emph{weighted majority} (WM) vote of the classifiers in $\cH$. We have so far explored two algorithms that generalize the zero-error strategy in different settings in an attempt to address the tradeoff: the relaxed CSS of Algorithm~\ref{relcssalg} and the relaxed enumeration strategy of Algorithm~\ref{relenumalg}. Reexamining these will motivate the new WM voting method that we cover in this section.

Algorithm~\ref{relcssalg} generalizes CSS to the non-realizable case - essential to address the tradeoff - by relaxing the strict training set consistency requirement on the version space in CSS, instead considering the lowest-$\iwerr{h}$ hypotheses in $\cH$. Independently, Algorithm~\ref{relenumalg} generalizes enumeration to allow some mistakes in return for fewer abstentions - directly addressing the tradeoff - by relaxing the unanimity requirement on the version space for a prediction. Instead, it makes predict-or-abstain decisions based on the margin of a vote taken among hypotheses in the version space.

Combining these two separate relaxation ideas, a sensible algorithm could make predict-or-abstain decisions based on the margin of a vote over the lowest-$\iwerr{h}$ hypotheses in $\cH$, instead of assessing strict unanimity among only hypotheses with $\iwerr{h} = 0$. This precise intuition can be realized by weighting each hypothesis $h$ according to $\iwerr{h}$ and then deciding whether to predict based on the WM vote. The resulting algorithm, described in~\cite{FMS04}, is shown as Algorithm~\ref{aggalg} for finite $\cH$.

\begin{algorithm}
\caption{Weighted Majority Voting with Abstention}
\label{aggalg}
\begin{algorithmic}[1]
\State Given: training set $S = \{(x_1, y_1), \dots, (x_m, y_m)\}$, failure probability $\delta$, tuning parameter $\theta \in (0, 1/2)$
\State $\eta \gets \ln(8\abs{\cH}) m^{\frac{1}{2} - \theta}, \;\;\Delta \gets 2 \sqrt{\frac{\ln (\sqrt{2}/\delta)}{m}} + \frac{\eta}{8m}$
\For {any unlabeled point $x$}
\State $\hat{L}_\eta (x) \triangleq \displaystyle \frac{1}{\eta} \ln \left( \frac{\displaystyle \sum_{h: h(x) = +1} e^{-\eta \iwerr{h}}}{\displaystyle \sum_{h: h(x) = -1} e^{-\eta \iwerr{h}}} \right)$
\If{$\abs{\hat{L}_\eta (x)} > \Delta$}
	\State Predict $\hat{y} = \sgn(\hat{L}_\eta (x))$
\Else
	\State Output $\perp$
\EndIf
\EndFor
\end{algorithmic}
\end{algorithm}

Here each hypothesis $h$ are weighted as $e^{-\eta \iwerr{h}}$, so that the high-weight hypotheses have low empirical error - an implicit encoding of the idea behind Algorithm~\ref{relcssalg}, because only the high-weight hypotheses participate much in the vote. $\eta > 0$ is a learning rate that controls how much the weighting deviates from the default counting measure on $\cH$; a larger training set leads to a larger $\eta$, implying greater confidence that the up-weighted hypotheses are indeed the ones with lowest true error $\err{h}$. (To see this, consider the limit $m \to \infty$, in which case $\forall h \in \cH: \iwerr{h} \to \err{h}$, and the weighting considers only $\argmin_{h \in \cH} \err{h}$ while ignoring all other hypotheses). It can also be seen that the $\hat{L}_\eta (x)$ thresholding achieves the same effect as Algorithm~\ref{relenumalg}: predict if the vote is enough of a landslide, and otherwise abstain.

As with the work we have previously examined, the theoretical guarantees proved for this algorithm include bounds on the abstain and error probabilities:

\begin{prop}[Corollary 1 in~\cite{FMS04}]
The error probability of Algorithm~\ref{aggalg} is at most $2 \err{h^*} + \ctO{m^{\theta - \frac{1}{2}}} + \delta$. Also, for $m \in \Omega \left( (\sqrt{\ln(1/\delta)} \ln (\abs{\cH}))^{1/\theta} \right)$, the abstain probability is at most $5 \err{h^*} + \cO{\frac{\sqrt{\ln(1/\delta)} + \ln (\abs{\cH})}{m^{\frac{1}{2} - \theta}}}$.
\end{prop}
\begin{proof}[Proof Sketch]
The proof idea for the abstain bound is fairly standard and relies on concentration of $\hat{L}_\eta (x)$ about $L_\eta (x)$, its non-empirical analogue in which $\iwerr{h}$ is replaced by $\err{h}$ in the calculations. The abstain threshold $\Delta$ therefore can be seen to quantify the finite-sample uncertainty in estimating $L_\eta (x)$ by $\hat{L}_\eta (x)$.

The error bound, however, is proved differently. It essentially argues that the vote is dominated by "good" hypotheses ($\err{h}$ close to $\err{h^*}$), and uses the tautology that there must be at least one such good hypothesis - $h^*$ itself - to argue that the vote too must perform well.

These arguments are quite general compared to other statistical learning arguments we have outlined, and hold equally well for the realizable and agnostic cases. They also hold for infinite $\cH$ with an additional explicit assumption that there is a nonzero mass of "good" hypotheses.
\end{proof}

These bounds are seen to be quite loose (in their dependence on $m$) compared to those in Sec.~\ref{tradeoffchar}. However, they do characterize the tradeoff, and exhibit a further interesting property: the error bound is completely independent of $\cH$, as the dependence on hypothesis class complexity is corralled in the abstain probability. By comparison, the standard "Occam's Razor" IID-learning bound on a vanilla binary classifier $\err{h^*} + \cO{\sqrt{\frac{\ln(\abs{\cH}/\delta)}{m}}}$ does depend, albeit weakly, on $\abs{\cH}$~\cite{BEHW87}. This is to our knowledge unique in the literature as another potential benefit of allowing the learner to abstain.

Though there is no explicit version space here, the predictions of all $h \in \cH$ are nevertheless necessary, so the same intractability issues that plague CSS and enumeration recur here. This can be mitigated, for instance by carefully constructing a manageably finite "$\epsilon$-cover" approximation of $\cH$~\cite{BB10}. However, it remains open how to use other tactics - such as the sampling and empirical error minimization ones previously mentioned - to address the problem.

The successive relaxations of unanimous-version-space arguments that we used to motivate this algorithm have the desirable consequence of making it much more robust. We illustrate this by noting that this WM algorithm is not foiled by the ill-matched geometry in the earlier construction of Figure 3.1. There are plenty of zero-error hypotheses in $\cH$ in that scenario, and they dominate the vote, potentially leading to zero error for high enough $m$.

Votes like this are well known to be desirable in another respect: they can perform far better than $h^*$ when the constituent hypotheses make diverse predictions, in effect compensating for each other's mistakes. This has further implications which are discussed in Sec.~\ref{futurework}.

Finally, the WM algorithm used here has several other suggestive interpretations (for instance, if the total $L^1$ normalized weight of hypotheses predicting $+1$ on $x$ is seen as an estimate of $\prp{y = +1 \mid x}{\cD}$, then $\hat{L}_\eta (x)$ is basically applying a logit transform). Other connections to previous work on similar weighted majorities in a variety of settings are too numerous to discuss here but highlight the richness of this approach~\cite{FS97, AHK12, LW94}.


\section{Connections to Active Learning}
\label{cssactive}
Active learning~\cite{S12} is a well-motivated supervised learning paradigm in which the learner attempts to minimize the number of labels it needs by requesting labels at its own discretion. As mentioned in Section~\ref{enumsec}, there is a natural correspondence between a $\perp$ output and a label request, linking abstain-based results with active learning. We have previously interspersed brief references to these connections where appropriate, but will now devote this section to elaborating on how the work outlined in this manuscript relates to active learning.

A major goal of active learning in the realizable case is to learn a hypothesis of error $\epsilon$ using only $\cO{\polylog (1/\epsilon)}$ labels. This represents exponentially fewer labels than required by standard supervised learning, for which $\Omega(1/\epsilon)$ labels are needed~\cite{D11}. This is an ambitious but sometimes achievable goal, as demonstrated by the canonical example of binary search as an active algorithm for learning a threshold on a line segment. It has also been proven for certain restricted special cases, like learning a linear separator when the data are uniformly distributed on a sphere~\cite{FSST97, DKM09}, but more general cases under which it is possible are much sought after in the community.


\subsection{Zero-Error Learning}

The online mistake-free enumeration algorithm (Alg.~\ref{enumalg}), developed for the KWIK setting, has been independently known as the CAL algorithm after the authors who first introduced it in~\cite{CAL94}, and is the forerunner of a significant strand of theoretically motivated active learning work~\cite{BBL06, BDL09, BHLZ10}. This is somewhat surprising because it is extremely conservative in settling for $\perp$, whereas active learning aims to \emph{minimize} the number of label requests. Indeed, one of the largest open problems in active learning is to design and analyze algorithms which are more "aggressive" in their label querying, rather than the more "mellow" CAL~\cite{D11}.

There is scant mention of this connection to active learning in the existing KWIK literature. However, we observed in Section~\ref{SLsec} that the CSS algorithm is exactly analogous to enumeration in the IID-data setting of statistical learning. Recent work~\cite{EW12} has shown a deep link between CSS and CAL (stated in terms of the VC dimension of $\cH$, a widely-used measure of the hypothesis class complexity, although other such measures would also suffice):
\begin{prop}[Paraphrase of Thm.9 from~\cite{EW12}]
Suppose $\cH$ is a hypothesis class with VC dimension $d < \infty$. Then if CSS can learn it with $\prp{\perp}{\cD} \in \cO{\frac{\polylog (m)}{m}}$, CAL will learn it using only $\cO{\polylog (1/\epsilon)}$ labels.
\end{prop}

This result motivates the further study of CSS to potentially yield highly desirable active learning results. A good beginning in this vein is also proved in~\cite{EW12}, using computational geometry tools from~\cite{EW10} to show that CAL requests $\cO{\polylog (1/\epsilon)}$ labels when $\cH$ is the class of linear separators in $\RR^d$ and $\cD$ is a mixture of a fixed finite number of Gaussians. There remain concerns about how such approaches scale with the dimension $d$ (an argument is made in~\cite{EW12} that CAL must request exponentially many labels in $d$), indicating that CAL has its limitations. More aggressive query strategies may be needed.


\subsection{Agnostic Learning}
\label{activenonzeroerror}

The idea of Algorithm~\ref{relcssalg} - to relax the strict training set consistency requirement of the version space in CSS and instead consider all hypothesis $h$ with low enough $\iwerr{h}$ - led to a major advance in active learning when it was successfully applied to CAL in~\cite{BBL06}.

The implementation trick outlined earlier for Algorithm~\ref{relcssalg} - minimizing the empirical risk of a hypothetical dataset as a means of detecting unanimity among hypotheses in a relaxed version space - has also appeared several times in the abstention-related literature and in active learning~\cite{EW10, SZB10, BHLZ10, DHM07, BBL06}. The "excess error" is assumed to be well-behaved in the analysis of the relaxed CSS algorithm; such assumptions on the same quantity have been partially studied in active learning as the Tsybakov noise conditions~\cite{T04}, under which efficient active learning algorithms with provably favorable properties have been derived. In fact, the excess error turns out to be fundamental to this line of active learning work, as made explicit recently in~\cite{ABE12}.


\section{Future Work and Conclusion}
\label{futurework}

The study of abstaining classifiers is far from complete, and there are several open areas to be addressed. Prominent among them is the study of abstaining classifiers without a version space. The version space idea has been central to much work so far because of its thoroughly-explored utility in the realizable case when zero error is allowed. This is somewhat of a toy case compared to most real applications, however, and necessitates constant efforts to make the resulting algorithms more tractable. It is also an unnecessarily brittle notion, as the construction of Figure 3.1 suggests. 

Another application-motivated open problem is to refine the insights of the work discussed here to be sensitive to the relative costs of abstaining and erring on a prediction. These costs generally drive the tradeoff in applications and could therefore be assumed to be known beforehand. An ambitious goal in this spirit is to devise algorithms which can precisely quantify their own confidence in a prediction, in a probabilistic manner.

Abstaining classifiers are also intimately connected to active learning, as discussed in Section~\ref{cssactive}. A range of ideas, from implementation tricks to algorithmic intuitions to the specific algorithms themselves, have been shown to recur in both fields. At present, there is very little work like~\cite{EW12} that attempts to explicitly unify the fields further. Both fields have a long history and have similar goals, yet there are several ideas, like relaxed enumeration in the KWIK setting and the disagreement coefficient in active learning~\cite{BBL06}, that are currently explored asymmetrically in one field. Progress on cross-pollinating these ideas would be of mutual enrichment to both areas.

Another area of potentially fruitful research (on which we have made progress) concerns improvements on Algorithm~\ref{aggalg}, the weighted aggregation algorithm. We observed that this algorithm can be seen as synthesizing intuitions from relaxed enumeration and relaxed CSS, and so we believe the classifier aggregation approach for abstention is worth deeper examination. In particular, majority votes are remarkably robust, and are known to perform better than any of their constituents when the constituent predictions are diverse~\cite{KW03}. Hypothesis weighting strategies that maximize this diversity could provably help performance~\cite{WLR06}, both for abstaining classifiers and in active learning.

We hope this survey provides some insight into recent approaches to designing classifiers that can abstain, which should see increased deployment in various applications over the coming years.


\subsubsection*{Acknowledgments}

I would like to thank my committee - David Kriegman, Sanjoy Dasgupta, and Kamalika Chaudhuri - for their time and feedback at various stages of the research exam process. Thanks are also due to Yoav Freund for first introducing me to this area and for various helpful suggestions.

\newpage
\bibliography{WriteupAbstain}{}

\begin{thebibliography}{BEHW87}

\bibitem[AADK13]{AADK13}
Jacob Abernethy, Kareem Amin, Moez Draief, and Michael Kearns.
\newblock Large-scale bandit problems and kwik learning.
\newblock In {\em ICML}, 2013.

\bibitem[ABE12]{ABE12}
Nir Ailon, Ron Begleiter, and Esther Ezra.
\newblock Active learning using smooth relative regret approximations with
  applications.
\newblock {\em Journal of Machine Learning Research - Proceedings Track},
  23:19.1--19.20, 2012.

\bibitem[AHK12]{AHK12}
Sanjeev Arora, Elad Hazan, and Satyen Kale.
\newblock The multiplicative weights update method: a meta-algorithm and
  applications.
\newblock {\em Theory of Computing}, 8(1):121--164, 2012.

\bibitem[BB10]{BB10}
Maria-Florina Balcan and Avrim Blum.
\newblock A discriminative model for semi-supervised learning.
\newblock {\em J. ACM}, 57(3), 2010.

\bibitem[BBL06]{BBL06}
Maria-Florina Balcan, Alina Beygelzimer, and John Langford.
\newblock Agnostic active learning.
\newblock In Cohen and Moore \cite{DBLP:conf/icml/2006}, pages 65--72.

\bibitem[BDL09]{BDL09}
Alina Beygelzimer, Sanjoy Dasgupta, and John Langford.
\newblock Importance weighted active learning.
\newblock In Danyluk et~al. \cite{DBLP:conf/icml/2009}, page~7.

\bibitem[BEHW87]{BEHW87}
Anselm Blumer, Andrzej Ehrenfeucht, David Haussler, and Manfred~K. Warmuth.
\newblock Occam's razor.
\newblock {\em Inf. Process. Lett.}, 24(6):377--380, 1987.

\bibitem[BHLZ10]{BHLZ10}
Alina Beygelzimer, Daniel Hsu, John Langford, and Tong Zhang.
\newblock Agnostic active learning without constraints.
\newblock In {\em NIPS}, pages 199--207, 2010.

\bibitem[CAL94]{CAL94}
David~A. Cohn, Les~E. Atlas, and Richard~E. Ladner.
\newblock Improving generalization with active learning.
\newblock {\em Machine Learning}, 15(2):201--221, 1994.

\bibitem[CBGO09]{CGO09}
Nicol{\`o} Cesa-Bianchi, Claudio Gentile, and Francesco Orabona.
\newblock Robust bounds for classification via selective sampling.
\newblock In Danyluk et~al. \cite{DBLP:conf/icml/2009}, page~16.

\bibitem[Cho70]{C70}
C.~Chow.
\newblock On optimum recognition error and reject tradeoff.
\newblock {\em IEEE Transactions on Information Theory}, 16(1):41--46, 1970.

\bibitem[CM06]{DBLP:conf/icml/2006}
William~W. Cohen and Andrew Moore, editors.
\newblock {\em Machine Learning, Proceedings of the Twenty-Third International
  Conference (ICML 2006), Pittsburgh, Pennsylvania, USA, June 25-29, 2006},
  volume 148 of {\em ACM International Conference Proceeding Series}. ACM,
  2006.

\bibitem[Das11]{D11}
Sanjoy Dasgupta.
\newblock Two faces of active learning.
\newblock {\em Theor. Comput. Sci.}, 412(19):1767--1781, 2011.

\bibitem[DBL09]{DBLP:conf/icml/2009}
Andrea~Pohoreckyj Danyluk, L{\'e}on Bottou, and Michael~L. Littman, editors.
\newblock {\em Proceedings of the 26th Annual International Conference on
  Machine Learning, ICML 2009, Montreal, Quebec, Canada, June 14-18, 2009},
  volume 382 of {\em ACM International Conference Proceeding Series}. ACM,
  2009.

\bibitem[DHM07]{DHM07}
Sanjoy Dasgupta, Daniel Hsu, and Claire Monteleoni.
\newblock A general agnostic active learning algorithm.
\newblock In John~C. Platt, Daphne Koller, Yoram Singer, and Sam~T. Roweis,
  editors, {\em NIPS}. Curran Associates, Inc., 2007.

\bibitem[DKM09]{DKM09}
Sanjoy Dasgupta, Adam~Tauman Kalai, and Claire Monteleoni.
\newblock Analysis of perceptron-based active learning.
\newblock {\em Journal of Machine Learning Research}, 10:281--299, 2009.

\bibitem[DLL09]{DLL09}
Carlos Diuk, Lihong Li, and Bethany~R. Leffler.
\newblock The adaptive {\it k}-meteorologists problem and its application to
  structure learning and feature selection in reinforcement learning.
\newblock In Danyluk et~al. \cite{DBLP:conf/icml/2009}, page~32.

\bibitem[EYW10]{EW10}
Ran El-Yaniv and Yair Wiener.
\newblock On the foundations of noise-free selective classification.
\newblock {\em Journal of Machine Learning Research}, 11:1605--1641, 2010.

\bibitem[EYW12]{EW12}
Ran El-Yaniv and Yair Wiener.
\newblock Active learning via perfect selective classification.
\newblock {\em Journal of Machine Learning Research}, 13:255--279, 2012.

\bibitem[FMS04]{FMS04}
Yoav Freund, Yishay Mansour, and Robert Schapire.
\newblock Generalization bounds for averaged classifiers.
\newblock {\em The Annals of Statistics}, 32:1698--1722, 2004.

\bibitem[FS97]{FS97}
Yoav Freund and Robert~E. Schapire.
\newblock A decision-theoretic generalization of on-line learning and an
  application to boosting.
\newblock {\em J. Comput. Syst. Sci.}, 55(1):119--139, 1997.

\bibitem[FSST97]{FSST97}
Yoav Freund, H.~Sebastian Seung, Eli Shamir, and Naftali Tishby.
\newblock Selective sampling using the query by committee algorithm.
\newblock {\em Machine Learning}, 28(2-3):133--168, 1997.

\bibitem[GF08]{GF08}
William Gasarch and Stuart Fletcher.
\newblock The egg game.
\newblock {\em www.cs.umd.edu/~gasarch/BLOGPAPERS/egg.pdf}, 2008.

\bibitem[KW03]{KW03}
Ludmila~I. Kuncheva and Christopher~J. Whitaker.
\newblock Measures of diversity in classifier ensembles and their relationship
  with the ensemble accuracy.
\newblock {\em Machine Learning}, 51(2):181--207, 2003.

\bibitem[KWA97]{KWA97}
Jyrki Kivinen, Manfred~K. Warmuth, and Peter Auer.
\newblock The perceptron algorithm versus winnow: Linear versus logarithmic
  mistake bounds when few input variables are relevant (technical note).
\newblock {\em Artif. Intell.}, 97(1-2):325--343, 1997.

\bibitem[Lit88]{L88}
Nick Littlestone.
\newblock Learning quickly when irrelevant attributes abound: A new
  linear-threshold algorithm.
\newblock In {\em Machine Learning}, pages 285--318, 1988.

\bibitem[LLWS11]{LLWS11}
Lihong Li, Michael~L. Littman, Thomas~J. Walsh, and Alexander~L. Strehl.
\newblock Knows what it knows: a framework for self-aware learning.
\newblock {\em Machine Learning}, 82(3):399--443, 2011.

\bibitem[LV06]{LV06}
L{\'a}szl{\'o} Lov{\'a}sz and Santosh Vempala.
\newblock Hit-and-run from a corner.
\newblock {\em SIAM J. Comput.}, 35(4):985--1005, 2006.

\bibitem[LW89]{LW89}
Nick Littlestone and Manfred~K. Warmuth.
\newblock The weighted majority algorithm.
\newblock In {\em FOCS}, pages 256--261. IEEE Computer Society, 1989.

\bibitem[LW94]{LW94}
Nick Littlestone and Manfred~K. Warmuth.
\newblock The weighted majority algorithm.
\newblock {\em Inf. Comput.}, 108(2):212--261, 1994.

\bibitem[Set12]{S12}
Burr Settles.
\newblock {\em Active Learning}.
\newblock Synthesis Lectures on Artificial Intelligence and Machine Learning.
  Morgan \& Claypool, 2012.

\bibitem[SS11]{SS11}
Istv{\'a}n Szita and Csaba Szepesv{\'a}ri.
\newblock Agnostic kwik learning and efficient approximate reinforcement
  learning.
\newblock {\em Journal of Machine Learning Research - Proceedings Track},
  19:739--772, 2011.

\bibitem[SS12]{S-S12}
Shai Shalev-Shwartz.
\newblock Online learning and online convex optimization.
\newblock {\em Foundations and Trends in Machine Learning}, 4(2):107--194,
  2012.

\bibitem[SZB10]{SZB10}
Amin Sayedi, Morteza Zadimoghaddam, and Avrim Blum.
\newblock Trading off mistakes and don't-know predictions.
\newblock In John~D. Lafferty, Christopher K.~I. Williams, John Shawe-Taylor,
  Richard~S. Zemel, and Aron Culotta, editors, {\em NIPS}, pages 2092--2100.
  Curran Associates, Inc., 2010.

\bibitem[Tsy04]{T04}
Alexandre Tsybakov.
\newblock Optimal aggregation of classiÞers in statistical learning.
\newblock {\em The Annals of Statistics}, 32:135--166, 2004.

\bibitem[WEY11]{WE11}
Yair Wiener and Ran El-Yaniv.
\newblock Agnostic selective classification.
\newblock In John Shawe-Taylor, Richard~S. Zemel, Peter~L. Bartlett, Fernando
  C.~N. Pereira, and Kilian~Q. Weinberger, editors, {\em NIPS}, pages
  1665--1673, 2011.

\bibitem[WLR06]{WLR06}
Manfred~K. Warmuth, Jun Liao, and Gunnar R{\"a}tsch.
\newblock Totally corrective boosting algorithms that maximize the margin.
\newblock In Cohen and Moore \cite{DBLP:conf/icml/2006}, pages 1001--1008.

\bibitem[WSDL09]{WSDL09}
Thomas~J. Walsh, Istvan Szita, Carlos Diuk, and Michael~L. Littman.
\newblock Exploring compact reinforcement-learning representations with linear
  regression.
\newblock In Jeff Bilmes and Andrew~Y. Ng, editors, {\em UAI}, pages 591--598.
  AUAI Press, 2009.

\end{thebibliography}
\bibliographystyle{alpha}

%
%
%
%

\end{document}